%% file: main.tex
\documentclass[letterpaper]{article}
\usepackage[preprint]{aaai2027} 
\usepackage[hyphens]{url} 
\usepackage{graphicx} 
\usepackage{natbib}  
\usepackage{caption}  
\usepackage{booktabs}

\usepackage{amsmath}

\DeclareMathOperator*{\argmin}{arg\,min}
\DeclareMathOperator*{\Regret}{reg}
\usepackage{amsfonts}
\usepackage{amssymb}
\usepackage{mathtools}

\usepackage{algorithm}
\usepackage[noend]{algpseudocode}

\definecolor{policysearchcolor}{HTML}{B35C1E}
\newcommand{\algline}[2]{l.~\ref{#2}}
\newcommand{\no}{{\color{Red}\ding{55}}}
\newcommand{\yes}{{\color{Green}\ding{51}}}

\usepackage[most]{tcolorbox}
\usepackage{amsthm}
\usepackage{thm-restate}

\usepackage[dvipsnames]{xcolor}
\usepackage{tikz}
\usetikzlibrary{patterns,arrows.meta}

\usepackage{nicefrac}
\usepackage{soul}
\usepackage{scalerel}
\usepackage{tikzsymbols}

\usepackage[capitalize,noabbrev]{cleveref}
\crefname{algorithm}{Alg.}{Alg.}

\theoremstyle{plain}

\newtheorem{lemma}{Lemma}

\theoremstyle{definition}

\newtheorem{problem}{Problem}

\newtheorem{example}{Example}

\usepackage{enumitem}

\newtcolorbox{problembox}{
    colback=white,
    colframe=black,
    boxrule=0.6pt,
    arc=0pt,
    left=6pt,
    right=6pt,
    top=6pt,
    bottom=6pt
}

\usepackage{xcolor}

\usepackage{pifont}

\title{Optimizing Minimax Regret in Uncertain MDPs with Small Sets of Policies}
\author{Sterre Lutz, Daniël Vos, Matthijs T.J. Spaan, Anna Lukina}
\affiliations{Delft University of Technology\\\{s.lutz, d.a.vos, m.t.j.spaan, a.lukina\}@tudelft.nl}
\begin{document}

\newcommand*{\MDP}{M}
\newcommand*{\Family}{\mathcal{M}}
\newcommand*{\States}{S}
\newcommand*{\Actions}{A}
\newcommand*{\TransitionFun}{T}
\newcommand*{\RewardFun}{R}
\newcommand*{\InitialDistr}{\mu}
\newcommand*{\IDFunction}{I}
\newcommand*{\InfoSpace}{\mathcal{I}}
\newcommand*{\Policy}{\pi}
\newcommand*{\Mapping}{\delta}
\newcommand{\Performance}{V}
\newcommand{\Cluster}{\Family}
\newcommand{\Partition}{\mathcal{P}}

\maketitle
\begin{abstract}
Sequential decision-making in real-world applications often involves uncertainty about the environment's model. Uncertain Markov decision processes (UMDPs) represent the possible environments as a set of MDPs with shared states and actions but potentially different transition probabilities and rewards. Optimizing a single policy across all possible MDPs may sacrifice performance, while preparing an individually optimized policy for every MDP may violate operational, regulatory, or interpretability constraints on the number of policies that can be prepared and deployed. We consider settings in which model uncertainty is resolved shortly before execution, allowing the most suitable policy to be selected from a limited set prepared in advance. We introduce $k$-adaptable policy synthesis, which optimizes such a set of $k$ policies under a minimax-regret objective. We prove that the problem is NP-hard and develop KAPS, an exact nested branch-and-bound algorithm with problem-specific bounds and heuristics. KAPS jointly optimizes which MDPs share a policy and the policies themselves. Experiments across various UMDP benchmarks show that the largest reduction in regret consistently occurs when increasing from one to two policies. In the single-policy setting, KAPS is competitive with existing methods in solution quality and proves optimality substantially more often.
\end{abstract}

\input{sections/introduction}
\input{sections/related-work}
\input{sections/problem-definition}
\input{sections/method}
\input{sections/experiments}
\input{sections/conclusion}

\bibliography{sources}

\clearpage

\appendix
\input{appendices/proofs/np-hardness}
\input{appendices/proofs/partition-equivalence}
\input{appendices/proofs/correctness-kaps}
\input{appendices/benchmark-details}
\input{appendices/extra_results}
\end{document}

%% file: sections/introduction.tex
\section{Introduction}
\label{sec:introduction}

Sequential decision-making often takes place under uncertainty about how an environment will behave. A decision maker must choose a strategy of consecutive actions---a \emph{policy}---whose consequences depend on stochastic transitions and rewards. In many applications, there is an additional layer of \emph{model uncertainty}: several probabilistic models of the environment are considered plausible, but it is initially unknown which model describes the true environment for the policy to be executed.

Markov decision processes (MDPs) are a standard modeling formalism for sequential decision-making under uncertainty~\cite{puterman2014markov}. An MDP consists of states connected by stochastic transitions: an agent selects an action, after which the environment samples the next state according to a probability distribution conditioned on the current state and action. The agent receives rewards and seeks a policy that maps states to actions and maximizes the expected cumulative discounted reward. In comparison, an uncertain MDP (UMDP) can represent the additional layer of the environment's model uncertainty by specifying a finite set of candidate MDPs, each describing a possible environment realization~\cite {xu2009parametric}.

\begin{figure}
    \centering
    \includegraphics[width=\linewidth]{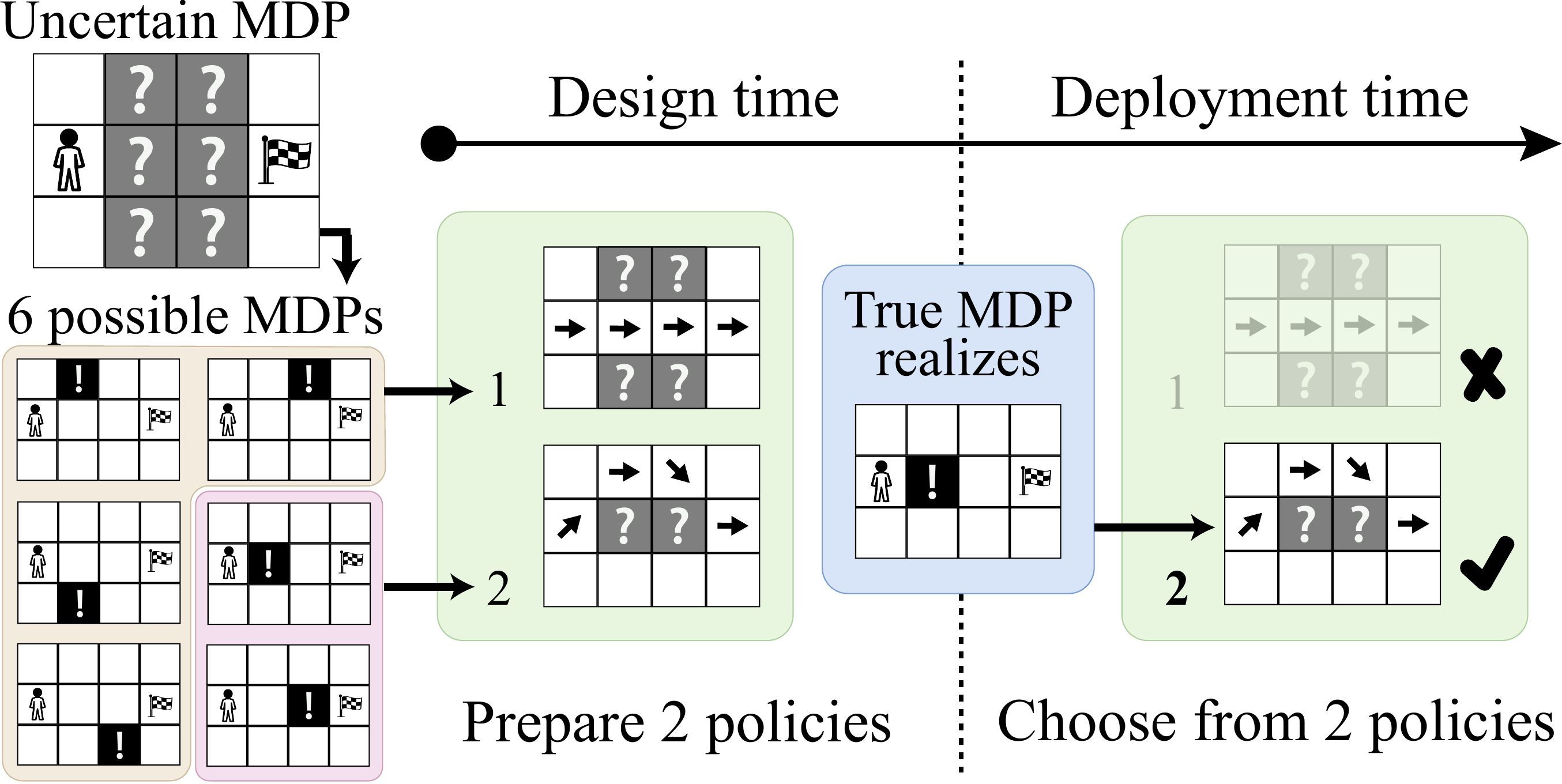}
    \caption{A $k$-adaptable policy for $k = 2.$ The uncertain MDP induces a set of possible MDPs (here "?" marks a possible obstacle location, "!" the true one; in general, the set varies in transition probabilities and rewards). At design time, two policies are prepared. Shortly before deployment, the realized environment is identified and the best-suited prepared policy is selected. A single policy may need to compromise between environments, whereas preparing one policy for every environment may be operationally infeasible.}
    \label{fig:illustration}
\end{figure}

We consider settings in which the model uncertainty is resolved shortly before deployment. Once identified, the realized environment could be handled by an optimal policy. In practice, however, policies may need to be designed, communicated, validated, or approved well in advance. Operational, regulatory, and interpretability constraints can therefore limit the number of distinct policies that can be prepared. The challenge is to prepare a small number of policies to select from, after
uncertainty resolves. \cref{ex:disaster} motivates our setting based on one of the benchmark problems we considered.

\begin{example}[Disaster Rescue]
A rescue team may prepare for several possible layouts of an area affected by a natural disaster, differing in which routes are blocked or which regions are difficult to traverse. An on-site assessment reveals the relevant details before the team begins its mission. Preparing one strategy per layout may be infeasible because instructions must be communicated, responders trained, and protocols validated or approved. The objective is a small policy set that performs well across all layouts. \Cref{fig:illustration} illustrates this setting.\label{ex:disaster}
\end{example}

Similar restrictions arise in healthcare, where each treatment protocol for different patient types may require clinical validation, and in maintenance scheduling, where deploying a policy can require investments or operational preparations to be made in advance. More generally, the objective arises in any application where model uncertainty resolves at deployment time, but preparing policies is costly.

Robust MDP policy synthesis generally assumes a conceptually different problem setting: model uncertainty does not resolve at deployment, and rather one policy is constructed for the whole UMDP~\cite{iyengar2005robust,el2005robust}. It has recently seen many efficient approaches~\cite{wolff2012robust,abateBestEffortPoliciesRobust2025,meggendorfer2025solving,schnitzer2026efficient}. These methods have primarily adopted a maximin-value objective, computing a single policy that maximizes performance against the worst-case MDP in the set. Minimax-regret approaches \citep{ahmedRegretBasedRobust2013,ahmedSamplingBasedApproaches2017,DBLP:conf/aaai/RigterLH21} instead measure a policy's loss relative to the environment-specific optimum, yielding less conservative solutions~\cite{regan2011robust,xu2009parametric}. In contrast, we formulate \emph{k-adaptable problem setting}, in which the model uncertainty resolves shortly before deployment, and extend the minimax-regret optimization to the \emph{$k$-adaptable policy synthesis}, which prepares $k$ policies for selection.

This gap between single and one-per-model solutions has been studied outside the MDP setting: work on finite adaptability and $k$-adaptability in robust optimization~\cite{bertsimasFiniteAdaptabilityMultistage2010,malaguti_k-adaptability_2022} shows that committing to a small fixed set of $k$ solutions before the uncertainty is resolved can close most of the gap to full adaptability at a fraction of the complexity. To our knowledge, no existing approach brings this idea to the general UMDP setting.

In this paper, we introduce the problem of $k$-adaptable minimax-regret policy synthesis for UMDPs, prove that it is NP-hard, and propose KAPS, a specialized exact branch-and-bound algorithm for solving it. KAPS jointly searches for a partition of the UMDP's underlying MDPs and one shared policy per subset, minimizing worst-case regret after the realized MDP is revealed. Problem-specific bounds and heuristics substantially reduce the resulting search space.

In the single-policy setting ($k=1$), KAPS consistently outperforms state-of-the-art UMDP planning approaches based on approximate dynamic programming~\cite{DBLP:conf/aaai/RigterLH21} and mixed-integer linear programming~\cite{ahmedRegretBasedRobust2013}. In the $k>1$ setting, across all benchmarks, most of the benefit is obtained from the first few policies: increasing from one to two policies yields the largest reduction in regret, often reducing it to near zero. This shows that even a small amount of adaptability can capture much of the value of tailoring a policy to the realized MDP.

In summary, our contributions are: (i) a formal definition of the $k$-adaptable minimax-regret policy synthesis problem and a proof of its NP-hardness;  (ii) the first exact algorithm to solve it; and (iii) an experimental evaluation on a range of UMDP benchmarks.

%% file: sections/related-work.tex
\section{Related Work}
\label{sec:related-work}
Although our work relates to several approaches to decision-making under model uncertainty, none directly applies to our problem setting. We discuss the key differences below.

\subsubsection{Minimax Regret in Uncertain MDPs}
Prior work on minimax regret in UMDPs focuses on synthesizing a single policy. \citet{ahmedRegretBasedRobust2013,ahmedSamplingBasedApproaches2017} propose an exact MILP formulation and scalable approximations, including Cumulative Expected Myopic Regret (CEMR). For stochastic-shortest-path UMDPs, \citet{DBLP:conf/aaai/RigterLH21} develop minimax value iteration and \(n\)-step approximations for dependent uncertainty; we refer to its three-step variant as cumulative regret (\textsc{CREG}) in the experimental evaluation. These methods provide relevant comparisons for single-policy synthesis, but do not address the synthesis of multiple policies.

\subsubsection{Finite Adaptability and Min--Max--Min Optimization}
Outside the MDP setting, finite adaptability restricts fully adaptive robust optimization to finitely many contingency plans \citep{bertsimasFiniteAdaptabilityMultistage2010}. In two-stage $k$-adaptability, $k$ candidate recourse decisions are prepared together with a common first-stage decision \citep{hanasusanto_k_2015, subramanyam_k-adaptability_2020}. Unlike our setting, this commits to a partial solution shared by all candidates. More directly related formulations prepare $k$ complete solutions and select one after uncertainty is revealed, optimizing either worst-case cost \citep{buchheim_minmaxmin_2017} or expected cost \citep{malaguti_k-adaptability_2022}. None of the discussed approaches address the synthesis of policies for UMDPs.

\subsubsection{Policy Synthesis for Sets of MDPs}
\citet{andriushchenko_policies_2025} construct policy trees that partition a set of MDPs (which they call a family) and assign each subset a common policy satisfying a reachability threshold, while identifying MDPs for which no such policy exists. Their abstraction-guided divide-and-conquer method recursively splits the MDPs without a prescribed limit on the resulting number of policies; compactness is promoted through splitting heuristics and post-processing rather than optimized exactly. In contrast, we prescribe the number $k$ of policies and search over alternative size-$k$ partitions to minimize worst-case regret over all MDPs.

\subsubsection{Adaptive Policies During Execution} 
Related work also considers settings in which the realized MDP is initially hidden and policies adapt during execution as observations reveal information about the true MDP. This includes memoryful or belief-based policies for UMDPs \citep{sharma2019robust} and multiple-environment MDPs \citep{raskin_multiple-environment_2014, bordais_multi-environment_2026}. In contrast, we assume that the realized MDP is identified before execution and select one of $k$ stationary memoryless policies. This keeps $k$ interpretable as a bound on the number of prepared state-dependent decision rules, rather than allowing a single policy to identify the MDP online and subsequently specialize its behavior further. \citet{luque-cerpaLearningContextualRuntime2026} instead take an existing ensemble of controllers and learn a contextual runtime monitor that continuously observes the operating context and selects a controller subject to safety requirements. We jointly synthesize the candidate policies to be selected from after uncertainty is resolved, rather than learning a runtime selector for a given set of policies.

%% file: sections/problem-definition.tex
\section{Preliminaries}
\label{sec:prelims}

We introduce the definitions needed to compare a policy across multiple possible MDPs. In particular, we define uncertain MDPs and evaluate policies by their worst-case loss relative to the optimal policy for the realized MDP.

\subsubsection{Markov Decision Processes (MDPs)}
A (discounted) Markov decision process is a tuple
\(
\MDP = (\States, \Actions, \TransitionFun, \RewardFun, \gamma, \InitialDistr),
\)
where $\States$ is a finite set of states, $\Actions$ is a finite set of actions, $\TransitionFun(\,\cdot \mid s,a)\in\Delta(\States)$ is the transition distribution for $(s,a)$, $\RewardFun : \States\times\Actions\times\States \to \mathbb{R}$ is a reward function, the scalar $\gamma\in(0,1)$ is the discount factor, and $\InitialDistr\in\Delta(\States)$ the initial distribution.

\subsubsection{Policies} A deterministic, memoryless policy is a function $\Policy:S\rightarrow A$ that selects an action based only on the current state. Its value in $\MDP$ is defined as the expected discounted return
$
    \Performance(\MDP, \Policy) = \mathbb{E}_{\MDP,\Policy}
    \left[
        \sum_{t=0}^{\infty}\gamma^t\RewardFun(s_t,\Policy(s_t),s_{t+1})
    \right],
$
where $s_0\sim\InitialDistr$ and $s_{t+1}\sim\TransitionFun(\cdot\mid s_t,\Policy(s_t))$. The optimal value of $\MDP$ is denoted by $\Performance^\star(\MDP)$ and a policy achieving it by $\Policy^\star$. 

\subsubsection{Uncertain MDPs} 
An uncertain MDP (UMDP) 
\(
\Family=\{\MDP_1,\dots,\MDP_n\}
\) 
is a finite set of MDPs sharing the state space $\States$, action space $\Actions$, and discount factor $\gamma$. The MDPs may differ in their transition kernels $\TransitionFun_i$, reward functions $\RewardFun_i$, and initial state distributions $\InitialDistr_i$. Each $\MDP_i\in\Family$ represents one possible realization of the environment. 

\subsubsection{Regret}
An optimal policy for one MDP may perform poorly for another. For an MDP $\MDP$, the regret of using policy $\Policy$ is the loss in value from the optimal policy for that MDP: $\Regret(\MDP,\Policy)\coloneqq\Performance^\star(\MDP)-\Performance(\MDP,\Policy).$ The regret of $\Policy$ on a UMDP $\Family$ is the worst-case regret across its MDPs. The minimum such worst-case regret is called \emph{minimax regret}, and a policy that achieves it a \emph{minimax-regret policy} 
\begin{equation}
\Policy^{\Regret}\in\argmin_{\Policy}\max_{\MDP\in\Family}(\Performance^\star(\MDP)-\Performance(\MDP,\Policy)).
\end{equation} 
In this paper, we restrict all synthesized policies to be deterministic and memoryless. We motivate this modeling choice in \cref{sec:problem}.

\section{$k$-Adaptable Policy Synthesis}
\label{sec:problem}
We consider settings where the realized MDP is unknown during synthesis but
revealed before execution. The decision maker can therefore prepare multiple policies in advance and select one \textit{after} the uncertainty has resolved. At the extreme, this means one could prepare a dedicated policy for each scenario. While this achieves optimal performance in every scenario, a large policy set may be impractical: each policy may need to be validated or approved, its possible behavior communicated to affected users, or its execution practiced by human operators. We therefore ask: \textit{how do we minimize worst-case regret with only $k$ prepared policies?} We formalize this as follows.

\begin{problembox}
\begin{problem}[$k$-Adaptable Policy Synthesis]\label{problem:synthesis}
Given a UMDP $\Family$ and an integer $k>0$, find a set of deterministic, memoryless policies $\{\Policy_1,\dots,\Policy_k\}$ that minimizes worst-case regret relative to the optimal value:
\begin{align}
\argmin_{\Policy_1,\dots,\Policy_k}
\: &
\max_{\MDP \in \Family} 
\:
\min_{i\in\{1,\dots,k\}}
\left(
\Performance^\star(\MDP)
-
\Performance(\MDP,\Policy_i)
\right).
\end{align}
\end{problem}
\end{problembox}

We restrict the candidate policies to be deterministic and memoryless. The memoryless restriction is central to the intended meaning of $k$: a history-dependent policy could use observed transitions to infer which MDP is active and subsequently branch into MDP-specific behavior, effectively encoding multiple policies in a single candidate. 

The determinism restriction is primarily computational. It yields a finite policy space that can be searched exactly by \textsc{KAPS}, aligns our setting with much of the existing UMDP policy-synthesis literature, and appears mild on our benchmarks, where deterministic policy sets already attain very low, and often zero, regret. Randomized policies may nevertheless improve performance on other instances, and extending \textsc{KAPS} to them remains future work.

\subsubsection{Complexity}
Whereas a single discounted MDP can be solved in polynomial time,
$k$-adaptable policy synthesis is NP-hard even for $k=1$ (proof in supplementary material).
\begin{restatable}[NP-hardness]{theorem}{nphardness}
\label{thm:np-hardness}
Solving $k$-adaptable policy synthesis (\cref{problem:synthesis}) is NP-hard, even for $k=1$.
\end{restatable}
\begin{proof}[Proof sketch]
We reduce 3-SAT to \cref{problem:synthesis} by constructing a UMDP that admits a zero-regret policy if and only if the given formula is satisfiable. 
\end{proof}
Thus, the computational difficulty does not arise only from assigning MDPs to multiple policies: even finding a single policy that minimizes worst-case regret across multiple MDPs is hard. This motivates the problem-specific exact search we describe in \cref{sec:method}.

%% file: sections/method.tex
\section{Algorithm for $k$-Adaptable Policy Synthesis}
\label{sec:method}

In this section, we introduce \textsc{KAPS}, an exact algorithm for
$k$-adaptable policy synthesis. \textsc{KAPS} relies on a partition-based
formulation that jointly determines which MDPs share a policy and
synthesizes one minimax-regret policy for each resulting subset. We first
show that solving this formulation yields an optimal solution to
$k$-adaptable policy synthesis. We then describe the nested
branch-and-bound searches used by \textsc{KAPS}, followed by the
problem-specific bounds and heuristics, and establish the correctness of
the algorithm.

\subsection{Partition-Based Formulation}
To solve \cref{problem:synthesis} exactly, we propose \textsc{KAPS}. As in $k$-adaptability approaches in optimization, assigning realizations to $k$ prepared solutions induces
a partition of the uncertainty set~\citep{bertsimasFiniteAdaptabilityMultistage2010}. Here, each subset in such a partition contains the MDPs selecting the same policy. These MDPs together form a smaller UMDP whose policy must be synthesized jointly.

This representation makes explicit the two coupled choices in \cref{problem:synthesis}: which MDPs share a policy and what policy is synthesized for each resulting UMDP. \textsc{KAPS} searches jointly for a partition of the UMDP and a minimax-regret policy for every smaller UMDP in that partition.

\begin{problem}[$k$-Partition Synthesis]\label{problem:partition}
    Given UMDP $\Family$ and integer $k>0$, find partition $\Partition=\{\Family_1,\ldots,\Family_k\}$ with $\bigcup_{i=1}^{k}\Family_i=\Family$, 
    and minimax-regret $\Policy^{\Regret}_1,\ldots,\Policy^{\Regret}_k$, s.t.
    \begin{align}
    \argmin_{\Partition}
    \: &
    \max_{i\in{1,\ldots,k}} 
    \max_{M\in\Family_i}
    \left(
    \Performance^\star(\MDP)
    -
    \Performance(\MDP,\Policy^{\Regret}_i)
    \right).
    \end{align}\label{prob:partition}
\end{problem}

Solving $k$-partition synthesis directly solves $k$-adaptable policy synthesis. The minimax-regret policies associated with the solution to \cref{problem:partition}, form an optimal set of prepared policies for \cref{problem:synthesis}. 

\begin{restatable}[Optimality Equivalence]
    {proposition}{partitionEquivalence}
\label{prop:partition-equivalence}
The optimal value obtained by solving \cref{problem:partition} equals the optimal value obtained by solving \cref{problem:synthesis}.
\end{restatable}

\begin{proof}[Proof sketch]
Any partition yields a set of $k$ policies for \cref{problem:synthesis} with the same or lower regret. Conversely, any $k$ policies induce a partition by assigning each MDP to its best policy; replacing each policy by a minimax-regret policy for its subset cannot
increase regret. Hence the optimal values coincide. A full proof is provided in the supplementary material.
\end{proof}

\subsection{Nested Branch-and-Bound}
\textsc{KAPS} performs $k$-partition synthesis (and consequently, \cref{problem:synthesis}) through two nested Branch-and-Bound (BnB) searches. BnB is a search strategy that recursively divides a solution space into smaller parts, creating a search tree in which each node represents a set of candidate solutions. Problem specific insights are used to lower bound all solutions represented by a node, permitting pruning of its children when the bound cannot improve the incumbent, the best solution found so far. Otherwise, the search continues until a leaf is reached, where the node represents a single solution that can be evaluated exactly.

\textsc{KAPS} nests two such searches to jointly optimize the partition and its policies. The outer BnB searches over partitions of the MDPs. Each search node represents all partitions completing some partial partition, in which some MDPs have already been assigned to subsets and the remaining MDPs are still unassigned. Branching assigns one additional unassigned MDP to a subset until a complete partition is obtained.

Evaluating the regret of a partition solution requires synthesizing a minimax-regret policy for the UMDPs it contains. This is performed by the inner BnB. Each search node represents a set of deterministic memoryless policies through restrictions on the actions available in each state. Branching iteratively restricts these actions until only a single policy remains. The inner search therefore determines the policies that are used for the UMDPs in the partition determined by the outer search. When $k=1$, the partition is fixed and only the inner search is required.
\DeclareRobustCommand{\stirling}{\genfrac\{\}{0pt}{}}
\subsection{Problem-Specific Reasoning and Heuristics}
Both search spaces in the nested BnB are combinatorial: a UMDP $\Family$ has $\stirling{|\Family|}{k}$ possible $k$-partitions (Stirling number of the second kind), and there are
$|\Actions|^{|\States|}$ deterministic memoryless policies. Unguided exploration of the search space is therefore infeasible even for moderate instances. 

To reduce the search effort, both BnB procedures use problem-specific reasoning and heuristics. Their overall structure is shown in \cref{alg:kaps,alg:subset-policy}. We introduce four key problem-specific components in both layers of the BnB:
(1) a \emph{bounding rule} that lower-bounds the best solution represented by a search node; 
(2) a \emph{branching heuristic} that determines how a node is divided into children; 
(3) an \emph{incumbent heuristic} that constructs promising complete solutions early; and
(4) a \emph{node-ordering heuristic} that determines the order in which nodes are explored.

\begin{algorithm}[tb]
\caption{\textsc{KAPS}}
\label{alg:kaps}
\input{figures/pseudocode_outer}
\end{algorithm}

\begin{algorithm}[tb]
\caption{Subprocedure {\color{Blue}\textsc{PolicySearch}}}
\label{alg:subset-policy}
\input{figures/pseudocode_inner}
\end{algorithm}

\subsubsection{\textsc{BoundPartition}}
For each subset in a partial partition, we first check whether any subset of that subset is in the cache, and take the largest available policy-search lower bound. Then, to strengthen the bound, we perform policy search for each partial partition subset and each pair of MDPs within the subset for a short duration. We again take the maximum of their lower bounds as a bound for the partition. Validity follows from monotonicity: adding MDPs to a UMDP cannot reduce minimax regret.
\subsubsection{\textsc{BranchPartition}}
To branch partitions efficiently, we aim to create partial partitions containing conflicting MDPs, i.e., MDPs that do not admit good policies when grouped together in a partition. This is useful, since it allows for quick pruning of the search space. We determine a conflict order once before the search, by performing \textsc{GuessPolicy} for each pair of MDPs, and sorting the MDPs by their worst pair's regret value in decreasing order. During the search, we select the next MDP to branch on in conflict order and create a branch for each subset to which this MDP is added.
\subsubsection{\textsc{GuessPartition}}
Starting with all MDPs arbitrarily distributed across the subsets, we repeatedly move a single MDP between subsets whenever this improves the partition regret estimate using \textsc{GuessPolicy}. Ties are broken in favor of more evenly sized subsets. Once no improving move remains, we evaluate the resulting partition using \textsc{GuessPolicy} or {\color{Blue}\textsc{PolicySearch}} (\cref{alg:subset-policy}) and use it as the initial incumbent.
\subsubsection{\textsc{NextPartition}}
Each search node is represented by a partial partition and is stored in the queue $Q$. We select the next node to expand as the one with the lowest (lower) bound.
\subsubsection{\textsc{BoundPolicy}}
We solve each $M\in\Family'$ through policy iteration under search node $N$'s action restrictions, finding the optimal value $V^\star_N(M)$. The highest regret relative to the unrestricted optimum,
$\max_{M\in\Family'}\bigl(V^\star(M)-V^\star_N(M)\bigr)$, lower bounds the regret of any policy represented by $N$.
\subsubsection{\textsc{GuessPolicy}}
We construct an average MDP like in \cite{DBLP:conf/aaai/RigterLH21} by averaging the transition kernels
$\TransitionFun_i$, reward functions $\RewardFun_i$, and initial-state
distributions $\InitialDistr_i$ over all $M_i\in\Family'$. The optimal
policy of this average MDP is suggested as a candidate policy.
\subsubsection{\textsc{BranchPolicy}}
We also use the policy proposed by \textsc{GuessPolicy} to guide branching.
Among the states whose actions are not yet fixed, we select the state
with the highest occupancy under this policy and branch on the action
chosen there. One child fixes this action, while the other excludes it.
\subsubsection{\textsc{NextPolicy}}
We select the next node with equal probability from two orderings: the node with the smallest upper bound, to prioritize promising candidate policies, or the node with the smallest lower bound, to prioritize regions that may still contain a low-regret policy.

\subsection{Correctness}
When run without a time limit, \textsc{KAPS} is guaranteed to have considered all possible partitions and corresponding policies and thus returns an optimal solution to \cref{problem:synthesis}. 

\begin{restatable}[Correctness of \textsc{KAPS}]{theorem}{correctness}
\label{thm:correctness}
\textsc{KAPS} terminates and returns an optimal solution to \cref{problem:synthesis}.
\end{restatable}

\begin{proof}[Proof outline]
\textsc{KAPS} terminates because both search trees are finite. Branching fully covers the solution space represented by each parent node. Both bounding rules are valid for the entire subtree rooted at a node: the \textsc{BoundPolicy} optimizes each MDP independently under the node restrictions, which is optimistic relative to requiring a fully shared policy, while \textsc{BoundPartition} uses smaller subsets of MDPs, which is optimistic compared to the level of uncertainty of the full subsets. Hence, pruning discards only nodes that cannot improve the incumbent. Since all unpruned leaves are evaluated exactly, the final incumbent is optimal.
\end{proof}

%% file: figures/pseudocode_outer.tex
\begin{algorithmic}[1]
\Require UMDP $\Family$ and policy-set size $k$
\Ensure Partition $\Partition_{\mathrm{best}}$, policies
$\Pi_{\mathrm{best}}$, and regret $\Regret_{\mathrm{best}}$

\State $N_0\gets((\emptyset,\dots,\emptyset),\Family)$
\State Initialize queue $Q$ with $N_0$
\State $(\Partition_{\mathrm{best}},
        \Pi_{\mathrm{best}},
        \Regret_{\mathrm{best}})
    \gets \textsc{GuessPartition}(\Family,k)$

\While{$Q$ is not empty \textbf{and}
       $L<\Regret_{\mathrm{best}}$}
    \State $N=(\Partition_N,\Family_{\mathrm{remaining}})
        \gets\textsc{NextPartition}(Q)$

    \If{$\Family_{\mathrm{remaining}}=\emptyset$}
        \ForAll{$\Family_i\in\Partition_N$}
            \State $(\Policy_i,\Regret_i)
                \gets{\color{Blue}\textsc{PolicySearch}}(\Family_i,\Regret_{\mathrm{best}})$
                \Statex \hfill $\triangleright$ see \cref{alg:subset-policy}
        \EndFor
        \State $\Regret_N\gets\max_i\Regret_i, \Pi_{N}\gets\{\Policy_i\}_{i=1}^k$

        \If{$\Regret_N<\Regret_{\mathrm{best}}$}
            \State $(\Partition_{\mathrm{best}},
                    \Pi_{\mathrm{best}},
                    \Regret_{\mathrm{best}})
                \gets
                (\Partition_N,\Pi_N,\Regret_N)$
        \EndIf

    \ElsIf{$\textsc{BoundPartition}(N)
        <\Regret_{\mathrm{best}}$}
        \State $\MDP\gets \textsc{BranchPartition}(N)$
        \State $Q \gets Q \cup \{\Partition_N \text{ with $\MDP$ added to $\Family_i$}\}_{i=1}^k$
    \EndIf
    \State $L \gets \min_{N'\in Q} \textsc{BoundPartition}(N')$
\EndWhile

\State \Return
$(\Partition_{\mathrm{best}},
  \Pi_{\mathrm{best}},
  \Regret_{\mathrm{best}})$
\end{algorithmic}

%% file: figures/pseudocode_inner.tex
\begin{algorithmic}[1]
\Require Subset $\Family'\subseteq\Family$ and regret cutoff
    $\Regret_{\mathrm{cutoff}}$
\Ensure Policy $\Policy_{\mathrm{best}}$ and value $U$, where
$\Policy_{\mathrm{best}}=\bot$ if no policy has regret below
$\Regret_{\mathrm{cutoff}}$

\State Let $N_0$ allow every action in every non-terminal state \label{line:root}
\State Initialize queue $Q$ with $N_0$
\State $\Policy_{\mathrm{best}}\gets\bot$,
       $U\gets\Regret_{\mathrm{cutoff}}$, and $L\gets 0$

\While{$Q$ is not empty \textbf{and} $L<U$}
    \State $N\gets\textsc{NextPolicy}(Q)$\label{line:pop}

    \If{$N$ represents a complete policy $\Policy$}\label{line:leaf}
        \If{$\max_{\MDP\in\Family'}\Regret(\MDP,\Policy)<U$}
            \State $\Policy_{\mathrm{best}}\gets\Policy, U\gets\max_{\MDP\in\Family'}\Regret(\MDP,\Policy)$
        \EndIf
    \ElsIf{$\textsc{BoundPolicy}(N)<U$}\label{line:bound}
        \State $\Policy\gets\textsc{GuessPolicy}(N)$
        \If{$\max_{\MDP\in\Family'}\Regret(\MDP,\Policy)<U$}
            \State $\Policy_{\mathrm{best}}\gets\Policy$,
             $U\gets\max_{\MDP\in\Family'}\Regret(\MDP,\Policy)$
        \EndIf

        \State $(s,a)\gets\textsc{BranchPolicy}(N)$\label{line:branch}
        \State $Q\gets Q\cup
            \{N+\textsc{force}(s,a),\,N+\textsc{forbid}(s,a)\}$\label{line:queue}
    \EndIf

    \State $L\gets
        \min_{N'\in Q}\textsc{BoundPolicy}(N')$
\EndWhile

\State \Return $(\Policy_{\mathrm{best}},U)$
\end{algorithmic}

%% file: sections/experiments.tex
\section{Experiments}
\label{sec:experiments}
We evaluate both the practical value of $k$-adaptable policy synthesis and the effectiveness of \textsc{KAPS} in solving it. Our experiments on a range of UMDP instances
support the following conclusions:
\begin{description}
    \item[C1.] KAPS outperforms existing methods on single-policy minimax-regret synthesis, particularly in optimality. 
    \item[C2.] A small policy set can capture much of the benefit of full adaptation, providing a favorable tradeoff between policy-set size and worst-case regret.
    \item[C3.] \textsc{KAPS} scales to instances with larger MDPs and greater numbers of MDPs in the UMDP. 
    \item[C4.] Each problem-specific component introduced in \cref{sec:method} contributes to the efficiency of KAPS.
\end{description}
Our open-source\footnote{https://github.com/SUMI-lab/KAPS} Python implementation of KAPS and all the baseline methods were run on a laptop with 24GB RAM and an Apple M4 Pro chip. 
\subsubsection{Benchmarks}
For our evaluation, we implement 6 kinds of UMDPs that are summarized below.
Disaster Rescue is a minimax-regret benchmark used in prior work \citep{ahmedRegretBasedRobust2013,ahmedSamplingBasedApproaches2017,DBLP:conf/aaai/RigterLH21}. 
Dynamic Power Management (D.P.M.) is adapted from an existing UMDP model \citep{andriushchenkoOracleGuidedApproachConstrained2025}. Maintenance is based on an established maintenance MDP formulation \citep{amari_cost-effective_2006}, to which we add uncertainty in the cost of 
labor and materials. 
Frozen Lake, Cliff Walking, and Taxi are adapted from the corresponding Gymnasium environments \citep{towers_gymnasium_2025}, to which we add uncertainty in, respectively, the hole layout, wind conditions, and passenger reliability.  
Complete UMDP specifications are provided in the supplemental material for reproducibility.
Unless stated otherwise, experiments use the smallest instance of
each benchmark listed in \cref{tab:single-policy-comparison}.

\subsubsection{C1: KAPS Outperforms Existing Methods}
While no existing method directly addresses $k$-adaptable policy synthesis, the problem reduces to single-policy minimax-regret synthesis for $k=1$, enabling a comparison with prior work. 

We compare KAPS at $k=1$ against existing methods:
\begin{itemize}[nolistsep]
    \item The exact MILP formulation of \cite{ahmedRegretBasedRobust2013}.
    \item CEMR with 3-step-options \cite{ahmedSamplingBasedApproaches2017}.
    \item CREG with 3-step-options \cite{DBLP:conf/aaai/RigterLH21}.
    \item Average MDP, optimizing the average of the MDPs.
    \item Best MDP, evaluating optimal policies for the individual MDPs on the full UMDP, and using the best one.
\end{itemize}
The Average and Best MDP heuristics were proposed in \cite{DBLP:conf/aaai/RigterLH21}. Only the MILP and KAPS solve the problem exactly and can prove optimality; the remaining methods are inexact heuristics, or approximations. For all methods, we use their stationary policy versions, i.e., policies that only depend on the state, not on time. 
We run each method on the benchmark instances with a time limit of five minutes. \Cref{tab:single-policy-comparison} reports the obtained minimax regret, runtime, and whether optimality was proven. For timed-out runs, it reports the remaining optimality gap where available: $(\Regret_{\mathrm{best}}-L)/\Regret_{\mathrm{best}}$, where $L$ is the lower bound and $\Regret_{\mathrm{best}}$ is the best known solution at timeout; the gap is zero when optimality is proven.

The inexact methods occasionally recover an optimal solution, but often return policies with considerably higher worst-case regret. KAPS obtains the lowest reported regret, on 11 of the 13 instances. On the remaining two instances, the MILP obtains a regret lower by only $0.01$. Among the exact methods, KAPS proves optimality on nine instances, compared with three for the MILP. Thus, for $k=1$, KAPS is competitive with the best existing method (MILP) in solution quality and substantially more effective at proving optimality.

\begin{table*}[t]
\centering
\label{tab:single-policy-comparison}
\input{figures/comparison-table}
\caption{Comparison of worst-case regret values, runtimes in seconds, and optimality gaps (\yes{} optimal, \no{} no bound, {\color{YellowOrange} gap}) of various methods from related work to \textsc{KAPS} with a runtime limit of 5 minutes on single policy optimization ($k{=}1$). MILP, CREG, and the Best MDP heuristic score well but are outperformed by \textsc{KAPS}. MILP is the only other method that can prove optimality, but \textsc{KAPS}' optimality gaps are consistently tighter.}
\end{table*}

\subsubsection{C2: Small Policy Sets Give Low or No Regret}
Our formulation of the $k$-adaptable policy synthesis problem is motivated by settings in which the realized MDP can be observed before execution, but the number of policies that can be prepared or maintained is constrained. We therefore show that constrained policy-set sizes provide a useful tradeoff between using one policy for the entire UMDP and preparing a separate optimal policy for every MDP.

For each benchmark, we run KAPS for policy-set sizes from $k=1$ to $k=|\Family|$. \Cref{fig:normalized-regret} reports the optimal minimax regrets, normalized by the regret obtained at $k=1$. At $k=|\Family|$, the problem trivially results in zero regret by separating all MDPs and assigning their optimal policies.

The largest reductions occur for the first few added policies. Moving from $k=1$ to $k=2$ eliminates regret for D.P.M. and Frozen Lake and substantially reduces it elsewhere; further increases in $k$ can still further reduce regret, but give diminishing returns. Thus, $k$ provides a meaningful tradeoff between the cost of maintaining additional policies and the benefit of adapting to the realized MDP.

\begin{figure}[t]
\centering
\includegraphics[width=\linewidth]{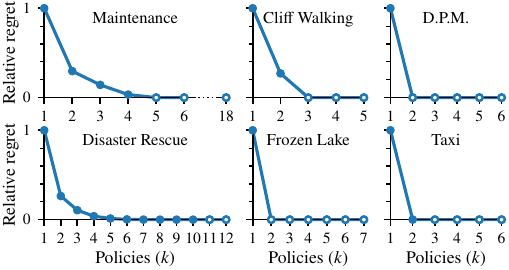}
\caption{Tradeoff between policy-set size $k$ and proven optimal minimax regret across benchmarks, normalized by the regret of the best single policy solution ($k{=}1$) for each benchmark. Unfilled markers indicate zero regret.}
\label{fig:normalized-regret}
\end{figure}

\subsubsection{C3: KAPS Scales to Larger Instances}
We investigate whether \textsc{KAPS} can still solve instances to proven
optimality as either the size of the underlying MDPs or the number of MDPs in the UMDP increases. We use the Maintenance benchmark, whose size can be varied systematically while retaining the same problem structure. 

We vary the number of states $|S|$ from 6 to 501---corresponding to more fine-grained machine degradation levels---and the number of MDPs $|\Family|$ from 5 to 100---corresponding to more random combinations of repair and replace costs. Other specifications of $\Family$ remain fixed, and $k=2$ is used throughout. \Cref{fig:scaling} reports the average runtimes (and 95\% confidence intervals) over 10 seeds required by \textsc{KAPS} to prove optimality for each instance.

\begin{figure}[t]
\centering
\includegraphics[width=\linewidth]{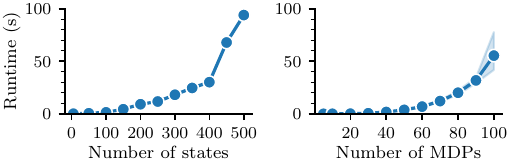}
\caption{Runtime required by \textsc{KAPS} to prove optimality on the Maintenance benchmark with $k{=}2$ when varying the number of states (left) and number of MDPs in the UMDP (right). The instances remain tractable even when scaled.}
\label{fig:scaling}
\end{figure}

Runtime initially grows gradually along both dimensions and increases more rapidly for the largest instances. Nevertheless, \textsc{KAPS} proves optimality for all tested instances within approximately 100 seconds, including MDPs with nearly 500 states and UMDPs containing 100 MDPs.

\subsubsection{C4: All Problem-Specific Components Contribute}
We assess whether the problem-specific components introduced in \cref{sec:method} improve efficiency over a generic BnB implementation, and whether their contributions are consistent across benchmarks. We perform cumulative ablations starting from a baseline that omits the bounding and incumbent heuristics. The baseline policy search branches on a uniformly random eligible state--action pair, whereas the partition search considers MDPs in input order. In both searches, the next node is selected uniformly at random from the queue.

For the policy-specific components, we set $k{=}1$ and successively enable \textsc{BoundPolicy}, \textsc{BranchPolicy}, \textsc{GuessPolicy}, and \textsc{NextPolicy}. For the partition-specific components, we set $k{=}2$ and successively enable \textsc{BoundPartition}, \textsc{BranchPartition}, \textsc{GuessPartition}, and \textsc{NextPartition}, while keeping all policy-specific components enabled. Each step therefore measures the marginal contribution of a rule given those enabled before it. Runtimes are normalized by the runtime of the full algorithm, and runs exceeding $100\times$ that runtime are treated as timeouts. We omit D.P.M.\ because its runtimes are too short to measure reliably, and Taxi because its runtime is too long without all components enabled.

\begin{figure}[t]
    \centering
    \includegraphics[width=\linewidth]{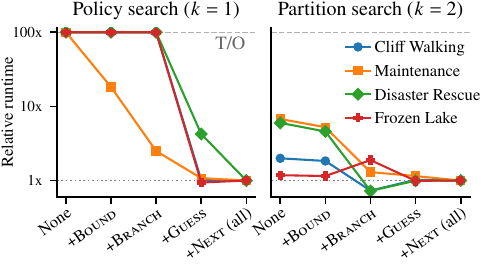}
    \caption{Runtime of the cumulative component ablations, normalized by the runtime of the full configuration for each benchmark. Timeout is set at $100\times$. The left panel adds policy-search components for listed benchmarks with $k{=}1$; the right panel adds partition-search components for benchmarks with $k{=}2$, with all policy-search components enabled.}
    \label{fig:ablation}
\end{figure}

\Cref{fig:ablation} shows that all components improve runtime on at least one benchmark. Policy-search components are collectively essential, with \textsc{GuessPolicy} giving the broadest gains. In partition search, \textsc{BranchPartition} has the largest overall effect, while \textsc{GuessPartition} and \textsc{NextPartition} help on selected benchmarks.
No partial configuration dominates, supporting the design choices of KAPS.

%% file: figures/comparison-table.tex
\setlength{\tabcolsep}{1.6pt}
\begin{tabular}{lrrrrrcrrcrrcrrcrrcrrc}
\toprule
 & & & & \multicolumn{3}{c}{Avg. MDP} & \multicolumn{3}{c}{Best MDP} & \multicolumn{3}{c}{CEMR} & \multicolumn{3}{c}{CREG} & \multicolumn{3}{c}{MILP} & \multicolumn{3}{c}{\textbf{\color{Blue}KAPS}} \\
 \cmidrule(lr){5-7} \cmidrule(lr){8-10} \cmidrule(lr){11-13} \cmidrule(lr){14-16} \cmidrule(lr){17-19} \cmidrule(lr){20-22}
Instance & $|S|$ & $|A|$ & $|\mathcal{M}|$ & regr. & time & opt. & regr. & time & opt. & regr. & time & opt. & regr. & time & opt. & regr. & time & opt. & regr. & time & opt. \\
\midrule
D.P.M. & 18 & 3 & 6 & \textbf{5.92} & <1 & \no & \textbf{5.92} & <1 & \no & 31.5 & 12 & \no & \textbf{5.92} & 12 & \no & \textbf{5.92} & 300 & \no & \textbf{5.92} & <1 & \yes \\
Cliff Walking & 35 & 4 & 5 & 1.16 & <1 & \no & 1.22 & <1 & \no & .82 & 31 & \no & \textbf{.66} & 1 & \no & \textbf{.66} & 300 & \no & \textbf{.66} & <1 & \yes \\
Maintenance S & 6 & 4 & 18 & 580 & <1 & \no & \textbf{555} & <1 & \no & 561 & 261 & \no & \textbf{555} & 300 & \no & \textbf{555} & <1 & \yes & \textbf{555} & <1 & \yes \\
Maintenance M & 51 & 4 & 18 & 220 & <1 & \no & \textbf{203} & <1 & \no & \textbf{203} & 300 & \no & \textbf{203} & 300 & \no & \textbf{203} & 300 & {\scriptsize\color{YellowOrange}92.0\%} & \textbf{203} & <1 & \yes \\
Maintenance L & 501 & 4 & 18 & 63.3 & <1 & \no & \textbf{59.9} & 1 & \no & \textbf{59.9} & 300 & \no & \textbf{59.9} & 300 & \no & \textbf{59.9} & 300 & {\scriptsize\color{YellowOrange}97.8\%} & \textbf{59.9} & 9 & \yes \\
Disaster Rescue & 32 & 8 & 12 & .58 & <1 & \no & .49 & <1 & \no & .87 & 300 & \no & .57 & 16 & \no & .56 & 300 & \no & \textbf{.34} & 1 & \yes \\
Disaster Rescue L & 100 & 8 & 20 & 3.33 & <1 & \no & 1.58 & <1 & \no & 3.56 & 300 & \no & 1.18 & 169 & \no & 9.28 & 300 & \no & \textbf{1.06} & 300 & {\scriptsize\color{YellowOrange}14.2\%} \\
Frozen Lake 4x4 & 16 & 4 & 7 & \textbf{1.00} & <1 & \no & \textbf{1.00} & <1 & \no & \textbf{1.00} & <1 & \no & \textbf{1.00} & 37 & \no & \textbf{1.00} & <1 & \yes & \textbf{1.00} & <1 & \yes \\
Frozen Lake 5x5 & 25 & 4 & 5 & .92 & <1 & \no & .95 & <1 & \no & .97 & 2 & \no & .91 & 14 & \no & \textbf{.83} & 202 & \yes & \textbf{.83} & 91 & \yes \\
Frozen Lake 6x6 & 36 & 4 & 5 & .92 & <1 & \no & .91 & <1 & \no & .96 & 3 & \no & .91 & 20 & \no & \textbf{.69} & 300 & {\scriptsize\color{YellowOrange}4.3\%} & .70 & 300 & {\scriptsize\color{YellowOrange}5.7\%} \\
Frozen Lake 7x7 & 49 & 4 & 5 & .91 & <1 & \no & .93 & <1 & \no & .96 & 3 & \no & .93 & 53 & \no & \textbf{.87} & 300 & {\scriptsize\color{YellowOrange}14.9\%} & .88 & 300 & {\scriptsize\color{YellowOrange}15.9\%} \\
Frozen Lake 8x8 & 64 & 4 & 5 & .91 & <1 & \no & .92 & <1 & \no & .95 & 4 & \no & .92 & 47 & \no & \textbf{.84} & 300 & {\scriptsize\color{YellowOrange}22.6\%} & \textbf{.84} & 300 & {\scriptsize\color{YellowOrange}20.2\%} \\
Taxi & 501 & 7 & 6 & 35.4 & <1 & \no & 21.3 & <1 & \no & 21.3 & 300 & \no & 38.9 & 60 & \no & 21.3 & 300 & \no & \textbf{12.5} & 197 & \yes \\
\bottomrule
\end{tabular}

%% file: sections/conclusion.tex
\section{Conclusion and Outlook}
\label{sec:conclusion}
Many sequential decision-making problems involve model uncertainty that is resolved only shortly before execution. Although this enables adaptation, policies may need to be prepared, communicated, practiced, validated, or approved in advance, limiting how many can be deployed. We introduced $k$-adaptable policy synthesis as a middle ground: prepare $k$ policies in advance and select the most suitable one once the realized MDP is known. We formalized this tradeoff under a minimax-regret objective, proved the problem NP-hard, and proposed \textsc{KAPS}, an exact nested branch-and-bound algorithm that jointly optimizes which MDPs share a policy and the policies themselves. Across our benchmarks, the first few additional policies capture much of the benefit of full adaptation. For $k=1$, \textsc{KAPS} is competitive with the best existing method in solution quality and proves optimality substantially more often. Scaling and ablation experiments further show that our problem-specific components significantly reduce runtime and enable larger instances to be solved to optimality.

Several directions remain for future work. As mentioned in \cref{sec:problem}, exploiting relationships between solutions for consecutive values of $k$ could efficiently trace the Pareto front between policy-set size and minimax regret, allowing decision makers to choose their preferred tradeoff after solving. Improved partition branching and node selection could make runtime more reliable across UMDPs, while approximate or sampling-based variants could extend the approach to very large or continuously parameterized sets of MDPs. Industrial case studies could further assess the practical value of bounded policy adaptation.

%% file: appendices/proofs/np-hardness.tex
\section{Proof of \cref{thm:np-hardness}}
We prove that k-adaptable policy synthesis with minimax regret (\cref{problem:synthesis}) is an NP-hard problem, by showing that the established NP-hard problem \textit{3-SAT} is reducible to it. An illustrative example of the reduction is given in \cref{fig:reduction-example}.
\input{figures/SAT-reduction}

\nphardness*
\begin{proof}[Proof of \cref{thm:np-hardness}]
    We reduce from \textsc{3-SAT}. Let $\varphi = C_1 \land \cdots \land C_m$ be a 3-CNF formula over variables $X = \{x_1,\ldots,x_n\}$. 
    We construct a family of MDPs $\Family_\varphi$ with one MDP $M_j$ for each clause $C_j$. We then solve the $k=1$ minimax regret problem on $\Family_\varphi$ and answer that $\varphi$ is satisfiable if and only if the optimal regret is $0$.

    All MDPs share $\States = \{s_1,\ldots,s_n\}\cup\{s_\mathrm{sat}, s_\mathrm{unsat}\}$ and $\Actions = \{0, 1\}$. For a clause $C_j$, let
\begin{equation*}
    \chi_j(i,a) =
    \mathbf{1}\{(x_i \in C_j \wedge a=1)
    \vee
    (\neg x_i \in C_j \wedge a=0)\},
\end{equation*}
which is $1$ exactly when assigning $x_i$ the truth value represented by action $a$ satisfies a literal of $C_j$.

The MDP $M_j$ is deterministic. From a variable state $s_i$, if $\chi_j(i,a)=1$, the process moves to $s_\mathrm{sat}$. If $\chi_j(i,a)=0$, it moves to $s_{i+1}$ when $i<n$, and to $s_\mathrm{unsat}$ when $i=n$. Both $s_\mathrm{sat}$ and $s_\mathrm{unsat}$ are absorbing. A reward is obtained only when entering $s_{\mathrm{sat}}$ from a different state.

    A policy $\Policy$ defines a truth assignment $\tau_\Policy$, where $\tau_\Policy(x_i)=\Policy(s_i)$ for all $i\in \{1, \dots,n\}$. We now show that for every clause $C_j$, $V_j^\pi(s_1)=1$ if and only if the assignment $\tau_\pi$ satisfies $C_j$.

    First, suppose that $\tau_\pi$ satisfies $C_j$. Then there exists an index $i$ such that $\chi_j(i,\pi(s_i))=1$. Let $i^\star$ be the smallest such index. For every $\ell<i^\star$, we have $\chi_j(\ell,\pi(s_\ell))=0$, so the trajectory continues from $s_\ell$ to $s_{\ell+1}$. Hence the trajectory reaches $s_{i^\star}$. Since $\chi_j(i^\star,\pi(s_{i^\star}))=1$, the transition from $s_{i^\star}$ goes to $s_{\mathsf{sat}}$ and yields reward $1$. Therefore $V_j^\pi(s_1)=1$.

    Second, suppose $V_j^\pi(s_1)=1$. Then, the trajectory must transition from some variable state $s_i$ to $s_{\mathsf{sat}}$. By construction, this implies $\chi_j(i,\pi(s_i))=1$. Therefore a literal of $C_j$ is true under $\tau_\pi$, so $\tau_\pi$ satisfies $C_j$.
    
    For every MDP $M_j$, the optimal value from the initial state is $V^*_j(s_1)=1$: since $C_j$ contains at least one literal, one can choose the corresponding action at its variable state and reach $s_\mathrm{sat}$. Thus the regret $V^*_j(s_1)-V^\Policy_j(s_1)=0$ if and only if $C_j$ is satisfied by $\tau_\Policy$. It follows that there exists a policy $\Policy$ with 
    \[
    \max_{j\in\{1,\dots,m\}} \bigl( V^*_j(s_1)-V^\Policy_j(s_1)\bigr)=0
    \]
    if and only if there exists an assignment $\tau$ satisfying all clauses of $\varphi$. Hence the constructed $k=1$ instance has optimal minimax regret $0$ if and only if $\varphi$ is satisfiable.

    The construction $\Family_\varphi$ has $m$ MDPs, each with $n+2$ states and two actions, and is therefore polynomial in the size of $\varphi$. Thus, deciding whether the optimal $k=1$ minimax regret is $0$ is NP-hard.
\end{proof}

%% file: figures/SAT-reduction.tex
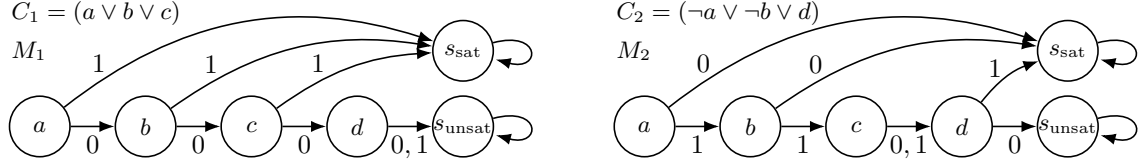
\begin{figure*}[t]
\centering

\def\xstep{1.4}      
\def\xterm{4.2}      
\def\yterm{1.0}      
\def\mdpshift{8.0}   

\begin{tikzpicture}[
    >=Latex,
    state/.style={draw, circle, minimum size=8mm, inner sep=0pt, font=\small},
    clause/.style={font=\small, anchor=west},
    mdp/.style={font=\small\bfseries, anchor=west},
    every path/.style={draw, semithick}
]

\node[clause] at (-0.5,1.5*\yterm) {$C_1=(a \lor b \lor c)$};
\node[mdp]    at (-0.5,1.0*\yterm) {$M_1$};

\node[state] (a1) at (0,0) {$a$};
\node[state] (b1) at (\xstep,0) {$b$};
\node[state] (c1) at (2*\xstep,0) {$c$};
\node[state] (d1) at (3*\xstep,0) {$d$};
\node[state] (sat1) at (4*\xstep,\yterm) {$s_{\mathrm{sat}}$};
\node[state] (unsat1) at (4*\xstep,0) {$s_{\mathrm{unsat}}$};

\path[->] (a1) edge node[below] {$0$} (b1);
\path[->] (b1) edge node[below] {$0$} (c1);
\path[->] (c1) edge node[below] {$0$} (d1);
\path[->] (d1) edge node[below] {$0,1$} (unsat1);

\path[->] (a1) edge[bend left=28] node[above, pos=0.10] {$1$} (sat1);
\path[->] (b1) edge[bend left=22] node[above, pos=0.16] {$1$} (sat1);
\path[->] (c1) edge[bend left=16] node[above, pos=0.28] {$1$} (sat1);

\path[->] (sat1) edge[loop right] ();
\path[->] (unsat1) edge[loop right] ();

\begin{scope}[xshift=\mdpshift cm]
\node[clause] at (-0.5,1.5*\yterm) {$C_2=(\neg a \lor \neg b \lor d)$};
\node[mdp]    at (-0.5,1.0*\yterm) {$M_2$};

\node[state] (a2) at (0,0) {$a$};
\node[state] (b2) at (\xstep,0) {$b$};
\node[state] (c2) at (2*\xstep,0) {$c$};
\node[state] (d2) at (3*\xstep,0) {$d$};
\node[state] (sat2) at (4*\xstep,\yterm) {$s_{\mathrm{sat}}$};
\node[state] (unsat2) at (4*\xstep,0) {$s_{\mathrm{unsat}}$};

\path[->] (a2) edge node[below] {$1$} (b2);
\path[->] (b2) edge node[below] {$1$} (c2);
\path[->] (c2) edge node[below] {$0,1$} (d2);
\path[->] (d2) edge node[below] {$0$} (unsat2);

\path[->] (a2) edge[bend left=28] node[above, pos=0.10] {$0$} (sat2);
\path[->] (b2) edge[bend left=22] node[above, pos=0.16] {$0$} (sat2);
\path[->] (d2) edge[bend left=16] node[above, pos=0.28] {$1$} (sat2);

\path[->] (sat2) edge[loop right] ();
\path[->] (unsat2) edge[loop right] ();
\end{scope}

\end{tikzpicture}
\caption{Reduction of the formula \(\varphi=(a \lor b \lor c)\land(\neg a \lor \neg b \lor d)\) to a family of MDPs. Each MDP \(M_j\) corresponds to one clause \(C_j\). A deterministic policy chooses action \(1\) (true) or \(0\) (false) at each variable state.}
\label{fig:reduction-example}
\end{figure*}

%% file: appendices/proofs/partition-equivalence.tex
\section{Proof of \cref{prop:partition-equivalence}}

We briefly reiterate the definitions and notation relevant to this proof. An indexed collection \(\Partition=(\Family_1,\dots,\Family_k)\) is a size-\(k\) partition of the UMDP \(\Family\) if and only if
\[
\bigcup_{i=1}^{k}\Family_i=\Family.
\]
For compactness, throughout this appendix let
\[
\Regret(\Family',\Policy)
\coloneq
\max_{\MDP\in\Family'}\Regret(\MDP,\Policy), \text{ and}
\]
\[
{\Regret}^\star(\Family')
\coloneq
\min_{\Policy} \Regret(\Family',\Policy).
\]
For a partition
$\Partition=(\Family_1,\ldots,\Family_k)$, let
\[
\Regret(\Partition)
\coloneq
\max_{i\in\{1,\ldots,k\}} {\Regret}^\star(\Family_i).
\]

\partitionEquivalence*
\begin{proof}
Let \(\Partition^\star=(\Family_1^\star,\dots,\Family_k^\star)\) minimize \(\Regret(\Partition)\), and let \(\Policy_i^\star\) be a minimax-regret policy for \(\Family_i^\star\). Construct the policy set
\[
\Pi^\star\coloneq\{\Policy_1^\star,\dots,\Policy_k^\star\}.
\]
For every \(\MDP\in\Family_i^\star\), the policy \(\Policy_i^\star\) is available in \(\Pi^\star\), and therefore
\[
\min_{j\in\{1,\dots,k\}}\Regret(\MDP,\Policy_j^\star)
\leq
\Regret(\MDP,\Policy_i^\star).
\]
Consequently,
\begin{align*}
\max_{\MDP\in\Family}\min_{j\in\{1,\dots,k\}}\Regret(\MDP,\Policy_j^\star)
&\leq\max_{i\in\{1,\dots,k\}}\max_{\MDP\in\Family_i^\star}\Regret(\MDP,\Policy_i^\star)\\
&=\Regret(\Partition^\star).
\end{align*}

Now consider any feasible policy set \(\Pi=\{\Policy_1,\dots,\Policy_k\}\). Assign each \(\MDP\in\Family\) to an index
\[
i(\MDP)\in\argmin_{j\in\{1,\dots,k\}}\Regret(\MDP,\Policy_j),
\]
breaking ties arbitrarily, and define
\[
\Family_i\coloneq\{\MDP\in\Family\mid i(\MDP)=i\}.
\]
The indexed subsets \(\Partition=(\Family_1,\dots,\Family_k)\) form a partition of \(\Family\). By the definition of optimal regret,
\[
{\Regret}^\star(\Family_i)
\leq
\max_{\MDP\in\Family_i}\Regret(\MDP,\Policy_i).
\]
Since \(\Partition^\star\) is an optimal partition,
\begin{align*}
\Regret(\Partition^\star)
&\leq\Regret(\Partition)\\
&\leq\max_{i\in\{1,\dots,k\}}\max_{\MDP\in\Family_i}\Regret(\MDP,\Policy_i)\\
&=\max_{\MDP\in\Family}\min_{j\in\{1,\dots,k\}}\Regret(\MDP,\Policy_j).
\end{align*}

Thus, the regret of \(\Pi^\star\) is no greater than that of any feasible policy set, so \(\Pi^\star\) is an optimal solution to \cref{problem:synthesis}.
\end{proof}

%% file: appendices/proofs/correctness-kaps.tex
\section{Proof of \cref{thm:correctness}}
Before proving the correctness of \textsc{KAPS}, we establish the correctness of its subprocedure \textsc{PolicySearch}
(\cref{alg:subset-policy}).

For compactness, throughout this appendix let
\[
\Regret(\Family',\Policy)
\coloneq
\max_{\MDP\in\Family'}\Regret(\MDP,\Policy), \text{ and}
\]
\[
{\Regret}^\star(\Family')
\coloneq
\min_{\Policy} \Regret(\Family',\Policy).
\]
For a partition
$\Partition=(\Family_1,\ldots,\Family_k)$, let
\[
\Regret(\Partition)
\coloneq
\max_{i\in\{1,\ldots,k\}} {\Regret}^\star(\Family_i) \text{ and }
\]
\[
{\Regret}^\star(\Partition)
\coloneq
\min_{\Partition} \Regret(\Partition).
\]

\subsection{Correctness of \textsc{PolicySearch}}
We first show that the subprocedure's bounding rule is valid (\cref{lem:bound-policy}) and that its branching rule correctly splits the represented policy space (\cref{lem:branch-policy}).

We introduce the following notation to reason about the set of policies represented by a node in the search procedure. For a policy-search node $N$, consisting of \textsc{force} and
\textsc{forbid} rules,
\[
\Pi_{N}
\coloneq
\left\{
\Policy:\States\to\Actions
\;\middle|\;
\begin{aligned}
\Policy(s)&=a
&&\forall\,\textsc{force}(s,a)\in N,\\
\Policy(s)&\neq a
&&\forall\,\textsc{forbid}(s,a)\in N
\end{aligned}
\right\}.
\]

\begin{lemma}[Validity of \textsc{BoundPolicy}]\label{lem:bound-policy}
For any policy-search node $N$, $\textsc{BoundPolicy}(N)$ lower-bounds the regret of every policy in $\Pi_{N}$.
\end{lemma}
\begin{proof}
By definition,
\[
\textsc{BoundPolicy}(N)
=
\max_{M\in\Family'}
\bigl(V^\star(M)-V^\star_N(M)\bigr).
\]
Take an arbitrary $\Policy\in\Pi_{N}$. By definition of an optimum, for every $M\in\Family'$ we have
$V^\star_N(M)\geq V(\Policy,M)$. Therefore,
\[
V^\star(M)-V^\star_N(M)
\leq
V^\star(M)-V(\Policy,M).
\]
Since the inequality holds pointwise for every $M\in\Family'$, taking the maximum over $M\in\Family'$ preserves it:
\begin{align*}
\max_{M\in\Family'}
\bigl(V^\star(M)-V^\star_N(M)\bigr)
&\leq
\max_{M\in\Family'}
\bigl(V^\star(M)-V(\Policy,M)\bigr)
\\
\textsc{BoundPolicy}(N)
&\leq
\Regret(\Family',\Policy).
\end{align*}
Since $\Policy\in\Pi_{N}$ was arbitrary, the bound is valid for every policy represented by $N$.
\end{proof}

\begin{lemma}[Correctness of policy branching]
\label{lem:branch-policy}
Let $N$ be a policy-search node, and let $(s,a)$ be any state--action pair for which there exist $\Policy,\Policy'\in\Pi_{N}$ such that $\Policy(s)=a$ and $\Policy'(s)\neq a$. Then the child nodes $N'$ and $N''$ added in \algline{alg:subset-policy}{line:queue} satisfy
\[
\Pi_{N'}\cup\Pi_{N''}=\Pi_{N},\quad\Pi_{N'}\subsetneq\Pi_{N}
\quad\text{and}\qquad
\Pi_{N''}\subsetneq\Pi_{N}.
\]
\end{lemma}

\begin{proof}
In \algline{alg:subset-policy}{line:queue} the nodes $N'=N+\textsc{force}(s,a)$ and $N''=N+\textsc{forbid}(s,a)$ are created. 

Take an arbitrary $\Policy\in\Pi_{N}$. If $\Policy(s)=a$, then $\Policy\in\Pi_{N'}$. Otherwise, $\Policy(s)\neq a$, and $\Policy\in\Pi_{N''}$. Thus, $\Policy \in \Pi_{N'}\cup \Pi_{N''}$ and 
\begin{equation}\label{eq:policy-branching-subset}
\Pi_{N} \subseteq \Pi_{N'}\cup \Pi_{N''}.    
\end{equation}

Conversely, take $\Policy\in\Pi_{N'}$. By construction, $N'$ retains all restrictions of $N$ and additionally requires $\Policy(s)=a$. Therefore, every policy represented by $N'$ is also represented by $N$, so $\Pi_{N'}\subseteq\Pi_{N}$. Similarly, $N''$ retains all restrictions of $N$ and additionally requires $\Policy(s)\neq a$, yielding $\Pi_{N''}\subseteq\Pi_{N}$. Thus,
\begin{equation}\label{eq:policy-branching-superset}
\Pi_{N'}\cup\Pi_{N''}\subseteq\Pi_{N}.
\end{equation}
Combining \cref{eq:policy-branching-subset,eq:policy-branching-superset} gives \[
\Pi_{N'}\cup \Pi_{N''}=\Pi_{N}.
\]

Finally, by assumption there exists $\Policy'\in\Pi_{N}$ with $\Policy'(s)\neq a$, so $\Policy'\notin\Pi_{N'}$ and $\Pi_{N'}\subsetneq\Pi_{N}$. Likewise, there exists $\Policy\in\Pi_{N}$ with $\Policy(s)=a$, so $\Policy\notin\Pi_{N''}$ and $\Pi_{N''}\subsetneq\Pi_{N}$.
\end{proof}

\begin{lemma}[Correctness of \textsc{PolicySearch}]\label{lem:correct-sub}
    If ${\Regret}^\star(\Family')<\Regret_{\mathrm{cutoff}}$, \textsc{PolicySearch} in \cref{alg:subset-policy} returns a policy $\Policy$ that achieves ${\Regret}^\star(\Family')$. Otherwise, it returns $\bot$.
\end{lemma}

\begin{proof}
First, we establish termination. The root node created in \algline{alg:subset-policy}{line:root} represents all deterministic stationary policies. This set is finite because the state and action spaces are finite. By \cref{lem:branch-policy}, each child represents a strict subset of the policies represented by its parent. By construction, \mbox{\textsc{BranchPolicy}} selects a state--action pair satisfying the premise of \cref{lem:branch-policy}. Hence, each branching step removes at least one represented policy, so every root-to-leaf path has length at most $|\Pi_{N_0}|-1$. Since the search tree is binary and has finite depth, it contains finitely many nodes. The procedure therefore terminates.

Suppose first that
${\Regret}^\star(\Family')<\Regret_{\mathrm{cutoff}}$. Initially, $U=\Regret_{\mathrm{cutoff}}$. By \cref{lem:branch-policy}, branching preserves every policy represented by an expanded node. A node is left unexpanded only if $\textsc{BoundPolicy}(N)\geq U$ in \algline{alg:subset-policy}{line:bound}. By \cref{lem:bound-policy}, no policy represented by such a node has regret strictly below $U$. Therefore, while $U>{\Regret}^\star(\Family')$, at least one node in the queue represents an optimal policy. Since the search terminates, either such a policy is eventually evaluated as a leaf in \algline{alg:subset-policy}{line:leaf}, or the incumbent is updated to the same optimal regret earlier by \textsc{GuessPolicy}. Thus, upon termination, the incumbent achieves ${\Regret}^\star(\Family')$.

Finally, suppose that ${\Regret}^\star(\Family')\geq\Regret_{\mathrm{cutoff}}$. No feasible policy has regret strictly below the initial value $U=\Regret_{\mathrm{cutoff}}$. Since both incumbent-update conditions require a strict improvement, $\Policy_{\mathrm{best}}$ remains equal to its initial value $\bot$. Therefore, the procedure returns $\bot$.
\end{proof}

\subsection{Correctness of \textsc{KAPS}}
We now establish the analogous properties for the branch-and-bound search over partitions done in \textsc{KAPS} (\cref{alg:kaps}). We first show that \textsc{BoundPartition} is valid (\cref{lem:bound-partition}) and that partition branching preserves all completions of a partial partition (\cref{lem:branch-partition}). Together with \cref{lem:correct-sub}, these properties are used to establish correctness of \textsc{KAPS} (\cref{thm:correctness}).

For a partition-search node $N=(\Partition_N,\Family_{\mathrm{remaining}})$, where $\Partition_N=(\Family_1,\ldots,\Family_k)$, let $\mathfrak{P}_N$ denote the set of complete partitions extending $\Partition_N$:
\[
\mathfrak{P}_N
\coloneq
\left\{
(\Family'_1,\ldots,\Family'_k)
\;\middle|\;
\begin{aligned}
&\Family_i\subseteq\Family'_i &&\forall i,\\
&\bigcup_{i=1}^k\Family'_i=\Family,\\
&\Family'_i\cap\Family'_j=\emptyset &&\forall i\neq j
\end{aligned}
\right\}.
\]

\begin{lemma}[Validity of \textsc{BoundPartition}]
\label{lem:bound-partition}
For any partition-search node $N$, $\textsc{BoundPartition}(N)$ lower-bounds the regret of every partition in $\mathfrak{P}_N$.
\end{lemma}

\begin{proof}
By definition, \textsc{BoundPartition} takes the largest regret bound over all subsets of the current subsets that are stored in the cache:
\[
\textsc{BoundPartition}(N)
=
\max_{i\in\{1,\ldots,k\}}
\max_{\substack{\Family''\subseteq\Family_i}}
{\Regret}^\star(\Family'').
\]

Let $\Partition'=(\Family'_1,\ldots,\Family'_k)\in\mathfrak{P}_N$ be an arbitrary partition represented by $N$. For every subset $\Family''$ considered above, we have $\Family''\subseteq\Family_i\subseteq\Family'_i$. Therefore, 
\[
{\Regret}^\star(\Family'')
\leq
{\Regret}^\star(\Family'_i)
\leq
\Regret(\Partition').
\]
The first inequality follows because enlarging the set of MDPs served by a single policy cannot decrease its optimal minimax regret, and the second follows from 
\[
\Regret(\Partition')
=
\max_{j\in\{1,\ldots,k\}}
{\Regret}^\star(\Family'_j).
\]

Thus, every quantity considered by \textsc{BoundPartition} lower-bounds $\Regret(\Partition')$. Taking their maximum preserves the inequality, so
\[
\textsc{BoundPartition}(N)
\leq
\Regret(\Partition').
\]
Since $\Partition'\in\mathfrak{P}_N$ was arbitrary, the claim follows.
\end{proof}

\begin{lemma}[Correctness of partition branching]
\label{lem:branch-partition}
Let $N$ be a partition-search node, and let $N_1,\ldots,N_r$, with $r\leq k$, be the children obtained by assigning the selected unassigned MDP to the eligible subsets. Then every partition in $\mathfrak{P}_N$ is represented by at least one child, possibly after relabeling its subsets.
\end{lemma}

\begin{proof}
Let $M\in\Family_{\mathrm{remaining}}$ be the MDP selected for branching, and let $\Partition'=(\Family'_1,\ldots,\Family'_k)\in\mathfrak{P}_N$.

Since $\Partition'$ is a partition of $\Family$, $M$ belongs to exactly one subset, say $\Family'_i$. Assigning $M$ to $\Family_i$ therefore yields a child that still represents $\Partition'$. If assignments to equivalent empty subsets are omitted, this holds up to relabeling.

Conversely, every child only adds an assignment to the partial partition represented by $N$. Thus, every partition represented by a child also belongs to $\mathfrak{P}_N$.
\end{proof}

\correctness*

\begin{proof}
The partition-search tree has depth at most $|\Family|$, since each branch assigns one previously unassigned MDP, and branching factor at most $k$, since there are at most $k$ subsets to which that MDP can be assigned. Hence, the tree contains finitely many nodes. Each leaf invokes \textsc{PolicySearch} finitely many times, and every invocation terminates by \cref{lem:correct-sub}. Therefore, \textsc{KAPS} terminates.

The root represents every complete partition of $\Family$. By \cref{lem:branch-partition}, branching preserves every represented partition. A node is pruned only if $\textsc{BoundPartition}(N)\geq\Regret_{\mathrm{best}}$. By \cref{lem:bound-partition}, no partition represented by such a node can improve the incumbent. Thus, while $\Regret_{\mathrm{best}}$ exceeds the optimal partition regret, some node in the queue represents an optimal partition.

When such a partition reaches a leaf, every subset $\Family_i$ satisfies
$
{\Regret}^\star(\Family_i)
<
{\Regret}_{\mathrm{best}}.
$
Hence, by \cref{lem:correct-sub}, each call to \textsc{PolicySearch} returns a policy attaining ${\Regret}^\star(\Family_i)$. The leaf is therefore evaluated at its true regret and updates the incumbent. Conversely, if a call returns $\bot$, then some subset has optimal regret at least $\Regret_{\mathrm{best}}$, so the partition cannot improve the incumbent.

Therefore, an optimal partition is eventually evaluated unless an equally good incumbent has already been found. Upon termination, $\Regret_{\mathrm{best}}$ equals the minimum partition regret. By \cref{prop:partition-equivalence}, the returned policies are optimal for \cref{problem:synthesis}.
\end{proof}

%% file: appendices/benchmark-details.tex
\section{Benchmark Specifications}
\label{app:benchmarks}

This appendix specifies the UMDPs used in the experiments. Unless stated otherwise, all unspecified transition probabilities are zero, rewards for transitions with zero probability are set to \(0\), and \(\gamma=0.999\). For benchmarks described through intermediate random outcomes, \(\TransitionFun_i\) is the transition distribution induced by those outcomes, and \(\RewardFun_i(s,a,s')\) is the expected transition reward conditional on reaching \(s'\).

\subsection{Disaster Rescue}
\label{app:benchmark-disaster-rescue}

Disaster Rescue is based on a minimax-regret UMDP benchmark used in prior work~\citep{ahmedRegretBasedRobust2013,ahmedSamplingBasedApproaches2017,DBLP:conf/aaai/RigterLH21}. The benchmark is a UMDP \(\Family\), where all MDPs share
\[
\begin{gathered}
\States=\{0,\dots,31\},\qquad
\Actions=\{0,\dots,7\},\\
\InitialDistr(24)=1,\qquad
\gamma=0.999.
\end{gathered}
\]
The terminal state is \(7\). The actions are the eight compass directions, ordered clockwise starting with up.

Each MDP \(M_i\) is specified by two obstacle states and one swamp state. The first obstacle state is selected from
\[
O^{(1)}=\{0,1,2,8,9,10,16,17,18\},
\]
the second obstacle state from
\[
O^{(2)}=\{22,23,30,31\},
\]
and the swamp state from
\[
W=\{4,5,6,12,13,14\}.
\]

For non-terminal states, the intended compass direction is taken with probability \(0.8\), while the two neighboring compass directions are each taken with probability \(0.1\). Moves that leave the grid are clipped to the nearest grid cell. If the resulting cell is not an obstacle, the transition enters that cell. If the resulting cell is an obstacle, the transition enters the obstacle with probability \(0.05\) and otherwise remains in the current state. The terminal state is absorbing.

The reward function is shared in form but depends on the swamp location of \(M_i\). Entering a non-swamp state has reward \(-0.5\). Entering the swamp state has reward sampled from \([-2,-1]\). The terminal state has reward \(0\) for all actions and next states.

\subsection{Dynamic Power Management}
\label{app:benchmark-dpm}

Dynamic Power Management is adapted from an existing UMDP model for constrained policy synthesis~\citep{andriushchenkoOracleGuidedApproachConstrained2025}. The benchmark is the UMDP
\[
\Family
=
\{M(\lambda,e):(\lambda,e)\in
\{0.10,0.25,0.45\}\times\{0.8,1.6\}\}.
\]
All MDPs share
\[
\begin{gathered}
\States=\{0,\dots,5\}\times
\{\mathsf{sleep},\mathsf{idle},\mathsf{active}\},\\
\Actions=\{\mathsf{sleep},\mathsf{idle},\mathsf{active}\},\qquad
\InitialDistr((0,\mathsf{idle}))=1,\qquad
\gamma=0.999.
\end{gathered}
\]

An MDP \(M_i=M(\lambda_i,e_i)\) differs only in the request-arrival probability \(\lambda_i\) and the energy-price multiplier \(e_i\). For a state \((q,m)\) and action \(a\), the next mode is \(a\). A request arrives with probability \(\lambda_i\). If \(a=\mathsf{active}\) and \(q>0\), one queued request is processed with probability \(0.9\); otherwise no request is processed. If \(x\in\{0,1\}\) denotes arrival and \(b\in\{0,1\}\) denotes service, the next queue length is
\[
q'=\min\{5,q-b+x\}.
\]
The transition probability \(\TransitionFun_i((q',a)\mid(q,m),a)\) is the total probability of all pairs \((x,b)\) that yield this value of \(q'\).

Rewards are negative costs. For an outcome \((x,b)\), define
\[
\ell=\mathbf{1}\{x=1\text{ and }q-b=5\},
\]
and
\[
c_i(q,m,a,x,b)
=
e_i c(a)(1+0.05q)
+
0.4\,\mathbf{1}\{a\neq m\}
+
12.0\,\ell,
\]
where
\(
c(\mathsf{sleep})=0.2,
c(\mathsf{idle})=1.0,\) and \(
c(\mathsf{active})=3.0.
\)
The reward \(\RewardFun_i((q,m),a,(q',a))\) is the negative expected value of \(c_i(q,m,a,x,b)\) over all outcomes \((x,b)\) that lead to \(q'\).

\subsection{Maintenance Scheduling}
\label{app:benchmark-maintenance}

Maintenance Scheduling is based on a condition-based maintenance MDP formulation~\citep{amari_cost-effective_2006}, to which we add uncertainty in repair and replacement costs. The benchmark is the UMDP
\begin{multline*}
    \Family=\{M(c_{\mathrm{repair}},c_{\mathrm{replace}}):(c_{\mathrm{repair}},c_{\mathrm{replace}})\\
    \in\{0.5,1.7,2.9\}\times\{1,2,3,4,5,6\}\}.
\end{multline*}

All MDPs share
\[
\begin{gathered}
\States=\{0,\dots,5\},\qquad
\Actions=\{\mathsf{wait},\mathsf{service},\mathsf{repair},\mathsf{replace}\},\\
\InitialDistr(0)=1,\qquad
\gamma=0.999.
\end{gathered}
\]
State \(0\) is healthy and state \(5\) is failed.

The transition kernel is shared by all MDPs. For \(s,s'\in\States\),
\[
\begin{aligned}
\TransitionFun(s'\mid s,\mathsf{wait})
&=
0.65\,\mathbf{1}\{s'=s\}{}+\\
&\quad
0.35\,\mathbf{1}\{s'=\min\{s+1,5\}\},\\[0.3em]
\TransitionFun(s'\mid s,\mathsf{service})
&=
0.88\,\mathbf{1}\{s'=s\}{}+\\
&\quad
0.12\,\mathbf{1}\{s'=\min\{s+1,5\}\},\\[0.3em]
\TransitionFun(s'\mid s,\mathsf{repair})
&=
0.75\,\mathbf{1}\{s'=0\}{}+\\
&\quad
0.25\,\mathbf{1}\{s'=s\},\\[0.3em]
\TransitionFun(s'\mid s,\mathsf{replace})
&=
\mathbf{1}\{s'=0\}.
\end{aligned}
\]
The operating cost is
\[
c_{\mathrm{op}}(s)
=
\begin{cases}
0.1+0.5s & \text{if } s<5,\\
8.0 & \text{if } s=5.
\end{cases}
\]
For \(M_i=M(c_{\mathrm{repair},i},c_{\mathrm{replace},i})\), the reward function is
\[
\begin{aligned}
\RewardFun_i(s,\mathsf{wait},s') &=
-c_{\mathrm{op}}(s),\\
\RewardFun_i(s,\mathsf{service},s') &=
-\bigl(c_{\mathrm{op}}(s)+0.4\bigr),\\
\RewardFun_i(s,\mathsf{repair},s') &=
-\bigl(c_{\mathrm{op}}(s)+c_{\mathrm{repair},i}\bigr),\\
\RewardFun_i(s,\mathsf{replace},s') &=
-c_{\mathrm{replace},i}.
\end{aligned}
\]

\subsection{Frozen Lake}
\label{app:benchmark-frozen-lake}

Frozen Lake is adapted from the corresponding Gymnasium environment~\citep{towers_gymnasium_2025}, with uncertainty in the hole layout. The benchmark is the UMDP
\[
\Family=\{M(H_j):j=0,\dots,6\},
\]
where \(H_j\subseteq\States\) is the set of hole states in \(M(H_j)\). States are indexed by
\[
s=x+4y,\qquad x,y\in\{0,\dots,3\}.
\]
All MDPs share
\[
\begin{gathered}
\States=\{0,\dots,15\},\qquad
\Actions=\{\mathsf{left},\mathsf{down},\mathsf{right},\mathsf{up}\},\\
\InitialDistr(0)=1,\qquad
\gamma=0.999.
\end{gathered}
\]
The goal state is \(15\).

For non-terminal states, transitions follow the slippery Frozen Lake dynamics: with probability \(1/3\) the agent moves in the selected direction, and with probability \(1/3\) each it moves in one of the two perpendicular directions. Moves that leave the grid keep the agent in the same state. Hole states and the goal state are absorbing. The reward function is shared by all MDPs:
\[
\RewardFun_j(s,a,s')=\mathbf{1}\{s'=15 \text{ and } s\neq 15\}.
\]
The hole sets are
\[
\begin{array}{ccl}
M(H_0) &:& H_0=\{13,14\},\\
M(H_1) &:& H_1=\{10,11\},\\
M(H_2) &:& H_2=\{3,9\},\\
M(H_3) &:& H_3=\{2,7\},\\
M(H_4) &:& H_4=\{4,5\},\\
M(H_5) &:& H_5=\{6,12\},\\
M(H_6) &:& H_6=\{11,14\}.
\end{array}
\]

\subsection{Cliff Walking}
\label{app:benchmark-cliff-walking}

Cliff Walking is adapted from the corresponding Gymnasium environment~\citep{towers_gymnasium_2025}, with uncertainty in wind conditions. The benchmark is the UMDP
\[
\Family
=
\{M(w):w\in
\{\mathsf{calm},\mathsf{north},\mathsf{east},\mathsf{south},\mathsf{west}\}\}.
\]
All MDPs share a \(5\times 7\) grid with states indexed by
\[
s=x+7y,\qquad x\in\{0,\dots,6\},\quad y\in\{0,\dots,4\}.
\]
The shared MDP components are
\[
\begin{gathered}
\States=\{0,\dots,34\},\qquad
\Actions=\{\mathsf{up},\mathsf{right},\mathsf{down},\mathsf{left}\},\\
\InitialDistr(14)=1,\qquad
\gamma=0.999.
\end{gathered}
\]
The goal state is \(20\), and the obstacle states are
\[
O=\{10,17,24\}.
\]
Obstacle states and the goal state are absorbing.

The shared movement function \(\operatorname{move}(s,d)\) returns the state reached by moving one cell from \(s\) in direction \(d\), with moves clipped at the grid boundary and moves into \(O\) replaced by staying in \(s\).

An MDP \(M_i=M(w_i)\) differs only in the wind direction \(w_i\). For non-absorbing states, the selected action is applied first:
\[
\bar{s}=\operatorname{move}(s,a).
\]
If \(w_i=\mathsf{calm}\), then
\[
\TransitionFun_i(s'\mid s,a)
=
\mathbf{1}\{s'=\bar{s}\}.
\]
Otherwise, the wind is applied after the action with probability \(0.25\):
\[
\TransitionFun_i(s'\mid s,a)
=
0.75\,\mathbf{1}\{s'=\bar{s}\}
+
0.25\,\mathbf{1}\{s'=\operatorname{move}(\bar{s},w_i)\}.
\]
The reward function is shared by all MDPs and depends only on the next state:
\[
\RewardFun_i(s,a,s')
=
\begin{cases}
0 & \text{if } s'=20,\\
-1 & \text{otherwise,}
\end{cases}
\]
with reward \(0\) for self-loops in absorbing states.

\subsection{Taxi}
\label{app:benchmark-taxi}

Taxi is adapted from Gymnasium Taxi-v4~\citep{towers_gymnasium_2025}, with weather-dependent dynamics, passenger reliability, hazard costs, and a skip action. The benchmark is the UMDP \(\Family=\{M_i\}_{i=1}^6\). All MDPs share the Taxi-v4 state encoding
\[
s=\operatorname{enc}(r,c,p,d),
\]
where \(r,c\in\{0,\dots,4\}\) are the taxi row and column, \(p\in\{0,\dots,4\}\) is the passenger location, and \(d\in\{0,\dots,3\}\) is the destination. We add one absorbing skip state, indexed by \(500\), so
\[
\States=\{0,\dots,500\}.
\]
The action space is
\[
\Actions=\{
\mathsf{south},\mathsf{north},\mathsf{east},\mathsf{west},
\mathsf{pickup},\mathsf{dropoff},\mathsf{skip}
\}.
\]
All MDPs share
\[
\InitialDistr(\operatorname{enc}(4,0,0,1))=1,
\qquad
\gamma=0.999.
\]
The hazard cells are
\[
H=\{(1,2),(2,2),(3,2)\}.
\]

For actions \(a\neq\mathsf{skip}\), the transition kernel of \(M_i\) is the Gymnasium Taxi-v4 kernel with
\[
\mathsf{is\_rainy}=\mathsf{true},\qquad
\mathsf{fickle\_passenger}=\mathsf{true},
\]
and MDP-specific rainy movement probability \(\rho_i\) and fickle-passenger probability \(\phi_i\). The skip action deterministically moves to the absorbing skip state:
\[
\TransitionFun_i(500\mid s,\mathsf{skip})=1
\quad\text{for all }s\neq 500.
\]
Successful dropoff states and the skip state are absorbing with reward \(0\).

Rewards are obtained by remapping the Gymnasium Taxi rewards. For \(a\neq\mathsf{skip}\), the standard Taxi reward \(-1\) is mapped to \(-c_{\mathrm{step},i}\), the reward \(20\) is mapped to \(r_{\mathrm{del},i}\), and the penalty \(-10\) is mapped to \(-10\). Transitions into a hazard cell receive an additional penalty \(c_{\mathrm{haz},i}\). The skip action has reward \(0\).

The six MDPs are
\[
\begin{array}{cclcccr}
\toprule
i & \text{Name} & \rho_i & \phi_i & c_{\mathrm{step},i}
& r_{\mathrm{del},i} & c_{\mathrm{haz},i}\\
\midrule
1 & \mathsf{clear\_stable}     & 0.98 & 0.0 & 1.0 & 25 & 0\\
2 & \mathsf{rainy\_stable}     & 0.75 & 0.0 & 1.2 & 25 & 4\\
3 & \mathsf{stormy\_stable}    & 0.55 & 0.0 & 1.5 & 25 & 10\\
4 & \mathsf{clear\_disrupted}  & 0.98 & 0.5 & 1.0 & 35 & 0\\
5 & \mathsf{rainy\_disrupted}  & 0.75 & 0.5 & 1.2 & 35 & 4\\
6 & \mathsf{stormy\_disrupted} & 0.55 & 0.5 & 1.5 & 35 & 10\\
\bottomrule
\end{array}
\]

%% file: appendices/extra_results.tex
\section{Additional Experimental Results}
The main paper includes results (\cref{tab:single-policy-comparison}) on a comparison between KAPS and MILP for $k=1$, both limited to a runtime of five minutes. \Cref{tab:single-policy-comparisons-appendix} contains additional measurements on the same experiments, at runtimes of one and twenty seconds. These measurements provided no additional insight and were omitted from the main paper.

\begin{table*}
\centering
\setlength{\tabcolsep}{4.8pt}
\begin{tabular}{l|rrrr|rrrr|rrrr}
\toprule
 & \multicolumn{4}{c}{1s runtime} & \multicolumn{4}{c}{20s} & \multicolumn{4}{c}{300s} \\
 \cmidrule(lr){2-5}\cmidrule(lr){6-9}\cmidrule(lr){10-13}
 & \multicolumn{2}{c}{MILP} & \multicolumn{2}{c}{\textbf{\color{Blue}KAPS}} & \multicolumn{2}{c}{MILP} & \multicolumn{2}{c}{\textbf{\color{Blue}KAPS}} & \multicolumn{2}{c}{MILP} & \multicolumn{2}{c}{\textbf{\color{Blue}KAPS}} \\
 \cmidrule(lr){2-3}\cmidrule(lr){4-5}\cmidrule(lr){6-7}\cmidrule(lr){8-9}\cmidrule(lr){10-11}\cmidrule(lr){12-13}
Instance & regret $\downarrow$ & opt. & regret & opt. & regret & opt. & regret & opt. & regret & opt. & regret & opt. \\
\midrule
D.P.M. & \textbf{5.92} & \no & \textbf{5.92} & \yes & \textbf{5.92} & \no & \textbf{5.92} & \yes & \textbf{5.92} & \no & \textbf{5.92} & \yes \\
Cliff Walking & .81 & \no & \textbf{.66} & \yes & \textbf{.66} & \no & \textbf{.66} & \yes & \textbf{.66} & \no & \textbf{.66} & \yes \\
Maintenance S & \textbf{555.4} & \yes & \textbf{555.4} & \yes & \textbf{555.4} & \yes & \textbf{555.4} & \yes & \textbf{555.4} & \yes & \textbf{555.4} & \yes \\
Maintenance M & \textbf{203.0} & \no & \textbf{203.0} & \yes & \textbf{203.0} & {\color{YellowOrange}98.4\%} & \textbf{203.0} & \yes & \textbf{203.0} & {\color{YellowOrange}92.0\%} & \textbf{203.0} & \yes \\
Maintenance L & 2578.5 & \no & \textbf{63.3} & \no & 60.1 & \no & \textbf{59.9} & \yes & \textbf{59.9} & {\color{YellowOrange}97.8\%} & \textbf{59.9} & \yes \\
Disaster Rescue & 45.1 & \no & \textbf{.34} & \yes & 4.53 & \no & \textbf{.34} & \yes & .56 & \no & \textbf{.34} & \yes \\
Disaster Rescue L & 494.6 & \no & \textbf{1.41} & {\color{YellowOrange}61.0\%} & 37.25 & \no & \textbf{1.22} & {\color{YellowOrange}32.0\%} & 9.53 & \no & \textbf{1.06} & {\color{YellowOrange}14.2\%} \\
Frozen Lake 4x4 & \textbf{1.00} & \yes & \textbf{1.00} & \yes & \textbf{1.00} & \yes & \textbf{1.00} & \yes & \textbf{1.00} & \yes & \textbf{1.00} & \yes \\
Frozen Lake 5x5 & .87 & {\color{YellowOrange}17.2\%} & \textbf{.85} & {\color{YellowOrange}8.2\%} & \textbf{.84} & {\color{YellowOrange}6.0\%} & \textbf{.84} & {\color{YellowOrange}4.8\%} & \textbf{.83} & \yes & \textbf{.83} & \yes \\
Frozen Lake 6x6 & \textbf{.84} & \no & \textbf{.84} & {\color{YellowOrange}26.2\%} & \textbf{.74} & {\color{YellowOrange}21.6\%} & \textbf{.74} & {\color{YellowOrange}10.8\%} & \textbf{.70} & {\color{YellowOrange}5.7\%} & \textbf{.70} & {\color{YellowOrange}5.7\%} \\
Frozen Lake 7x7 & .91 & {\color{YellowOrange}31.9\%} & \textbf{.88} & {\color{YellowOrange}15.9\%} & \textbf{.87} & {\color{YellowOrange}16.1\%} & .88 & {\color{YellowOrange}15.9\%} & \textbf{.87} & {\color{YellowOrange}14.9\%} & .88 & {\color{YellowOrange}15.9\%} \\
Frozen Lake 8x8 & .91 & {\color{YellowOrange}50.5\%} & \textbf{.90} & {\color{YellowOrange}31.1\%} & \textbf{.89} & {\color{YellowOrange}34.8\%} & .90 & {\color{YellowOrange}25.6\%} & \textbf{.84} & {\color{YellowOrange}23.8\%} & \textbf{.84} & {\color{YellowOrange}20.2\%} \\
Taxi & \textbf{21.3} & \no & \textbf{21.3} & {\color{YellowOrange}65.4\%} & 21.3 & \no & \textbf{17.2} & {\color{YellowOrange}57.2\%} & 21.3 & \no & \textbf{12.5} & \yes \\
\bottomrule
\end{tabular}
\caption{Comparison between KAPS and the exact MILP method when solving uncertain MDPs for a single robust policy. KAPS achieves high-quality solutions and bounds within 1 second of runtime and eventually proves optimality for most scenarios, except larger Disaster Rescue and Frozen Lake instances. Given 300 seconds of runtime, MILP achieves similar scores across many scenarios but performs significantly worse at proving optimality.}\label{tab:single-policy-comparisons-appendix}
\end{table*}